\documentclass[11pt]{article}

\usepackage[final]{acl}

\usepackage{hyperref}
\usepackage{url}
\usepackage[ruled,vlined,linesnumbered]{algorithm2e}
\usepackage{amssymb,amsmath}

\usepackage{amsthm}
\newtheorem{lemma}{Lemma}
\newtheorem{theorem}{Theorem}

\usepackage{xcolor} % add this in preamble
\usepackage{colortbl}
\usepackage{booktabs}
\usepackage{graphicx}
\usepackage{makecell}
\usepackage{multirow}
\usepackage{hhline}
\usepackage{subcaption}

\newcommand{\colorg}{\cellcolor{gray!15}}

\usepackage{pgfplots}
\pgfplotsset{compat=1.15}
\usepgfplotslibrary{
  statistics,
  colorbrewer,
  groupplots,
}
\usetikzlibrary{
  patterns,
  shapes.geometric,
  decorations.text,
  matrix,
  fit,
  backgrounds,
  positioning,
}
\tikzset{
  fignode/.style={
    outer sep=0.25em,
  }
}
\tikzset{
  framedfignode/.style={
    outer sep=0.25em,
    inner sep=0.5em,
    rounded corners,
    draw,
  }
}
\colorlet{plotColorNeutral}{gray}
\definecolor{plotColor1}{HTML}{f61a1c}
\definecolor{plotColor2}{HTML}{377eb8}
\definecolor{plotColor3}{HTML}{4daf4a}
\definecolor{plotColor4}{HTML}{984ea3}
\definecolor{plotColor5}{HTML}{FFFFCB}
\definecolor{plotColor6}{HTML}{1e90ff}
\definecolor{plotColor7}{HTML}{064E40}
\colorlet{plotColorNeutral*}{plotColorNeutral!40}
\colorlet{plotColor1*}{plotColor1!60}
\colorlet{plotColor2*}{plotColor2!60}
\colorlet{plotColor3*}{plotColor3!60}
\colorlet{plotColor4*}{plotColor4!60}
\colorlet{plotColor5*}{plotColor5!60}
\colorlet{plotColor6*}{plotColor6!60}
\colorlet{plotColor7*}{plotColor7!60}
\pgfplotsset{
    colormap={greenred}{HTML=(4daf4a) HTML=(e41a1c)},
    colormap={redgreen}{HTML=(e41a1c) HTML=(4daf4a)}
}

\newcommand{\ours}{\textsc{SpecGuard}}
\newcommand{\oursgreedy}{\textsc{SpecGuard Greedy}}

\usepackage[T1]{fontenc}
\usepackage[utf8]{inputenc}

\usepackage{microtype}

\usepackage{inconsolata}

\title{
From Tokens to Steps: Verification-Aware Speculative Decoding for Efficient Multi-Step Reasoning
}
\author{Kiran Purohit\thanks{Work done during internship at
Adobe Research.} \\ IIT Kharagpur \\ \texttt{kiran.purohit} \\\texttt{@kgpian.iitkgp.ac.in} \And Ramasuri Narayanam \\ Adobe Research \\\texttt{rnarayanam} \\\texttt{@adobe.com}\And Soumyabrata Pal \\ Adobe Research \\ \texttt{soumyabratap} \\ \texttt{@adobe.com}}

\begin{document}
\maketitle
\begin{abstract}

Speculative decoding (SD) accelerates large language model inference by allowing a lightweight draft model to propose outputs that a stronger target model verifies. However, its token-centric nature allows erroneous steps to propagate. Prior approaches mitigate this using external reward models, but incur additional latency, computational overhead, and limit generalizability.
We propose \textbf{\ours{}}, a verification-aware speculative decoding framework that performs \emph{step-level} verification using only model-internal signals. At each step, \ours{} samples multiple draft candidates and selects the most consistent step, which is then validated using an ensemble of two lightweight model-internal signals: (i) an attention-based grounding score that measures attribution to the input and previously accepted steps, and (ii) a log-probability-based score that captures token-level confidence. These signals jointly determine whether a step is accepted or recomputed using the target, allocating compute selectively. Experiments across a range of reasoning benchmarks show that \ours{} improves accuracy by $3.6\%$ while reducing latency by $\sim$11\%, outperforming both SD and reward-guided SD.

\end{abstract}

\section{Introduction}
\label{sec:intro}

Large language models (LLMs) have demonstrated a remarkable ability to solve complex multi-step reasoning problems across domains such as mathematics and knowledge-intensive tasks \cite{brown2020language,team2024gemini,hurst2024gpt}. However, their practical deployment is constrained by high inference costs, which limit scalability and real-time applicability \cite{patterson2021carbon}. {\em Reducing inference overhead without sacrificing accuracy has therefore become a central research challenge }\cite{xuthink,lin2024awq}.  

Speculative decoding (SD) \cite{leviathan2023fast} has emerged as a promising solution to accelerate inference, where a lightweight draft model generates candidate tokens, and a stronger target model verifies them. By offloading much of the token generation process to the smaller draft model, SD achieves significant latency reductions compared to decoding with the target model alone. Despite these gains, SD remains inherently token-centric, leading to critical limitations in reasoning tasks. Its strict unbiasedness requirement often rejects semantically correct draft tokens that have low probability under the target model, resulting in wasted computation and reduced efficiency \cite{judge_decoding_2025,holtzmancurious}. This rigidity limits speedups and makes it less effective for multi-step tasks.

Recent extensions of SD attempt to address this limitation. For example, reward-guided speculative decoding (RSD) \cite{rsd_paper_2025} introduces external pre-trained reward models (PRMs) to verify the correctness of the draft output. Although effective in improving reliability, it incurs substantial drawbacks. First, reliance on external verifiers significantly increases latency and compute overhead.
Second, pre-trained reward models are often specialized to specific tasks, making them difficult to generalize across diverse reasoning tasks.

This naturally leads to the central question driving our work:
{\em How can we design a speculative decoding framework that maintains accuracy in multi-step reasoning tasks while remaining cost-efficient and scalable, without relying on external verifier models?}

% This tension leads to the central question of this line of work: {\em How can we design a speculative decoding framework that preserves reasoning accuracy in multistep tasks while remaining cost-efficient and scalable, without dependence on external verifier models?}

% Our motivation is twofold:  
% \begin{itemize}
%     \item \textbf{Accuracy preservation.} Prevent error propagation across reasoning steps by ensuring that only reliable intermediate outputs are accepted, thereby maintaining correctness over entire reasoning chains.  
%     \item \textbf{Efficiency.} Achieve this verification in a lightweight and cost-effective manner, avoiding the use of large external models and preserving the latency benefits of speculative decoding.  
% \end{itemize}

We present \textbf{\ours{}}, a mathematically grounded, verification-aware framework for adaptive inference-time compute allocation.
%an adaptive and verifier-guided speculative decoding framework that introduces model-internal verifiers. 
The key intuition behind \ours{} is:
% can be summarized along two key dimensions:

{\em - Accuracy preservation:} Mitigate error propagation by ensuring trusted intermediate outputs are accepted, safeguarding correctness throughout the reasoning chain.

{\em - Efficiency:} Enable lightweight, cost-effective verification without relying on large external verifiers, reducing latency.

\ours{} integrates two lightweight verifiers derived directly from the model itself: (i) {\em Attention-based grounding verification}, which checks whether the generated step is properly grounded in the input context or previously validated steps, and (ii) {\em Log-probability-based verification}, which ensures confidence at the token level. These complementary signals are combined into an {\em ensemble verifier} that adaptively decides whether to accept draft outputs or invoke the target model. Furthermore, we introduce a novel {\em self-consistency selector} that identifies the most semantically consistent step from multiple sampled draft candidates. 
% Together, these innovations allow \ours{} to preserve accuracy in multi-step reasoning while reducing latency and compute overhead.  
To summarize, our key contributions are:  

    (1) We propose \textit{\ours{}}, a novel framework that integrates model-internal verifiers with adaptive inference-time compute allocation, improving reliability without the need for external reward models.
    
    (2) We introduce a novel \textit{self-consistency selector} that identifies the representative step from multiple sampled draft candidates.
    
    (3) Extensive experiments on various reasoning benchmarks show that \ours{} improves accuracy by up to 3.6\% while reducing latency by $\sim$11\% compared to state-of-the-art methods, making it both effective and efficient for real-world deployment.
    % demonstrating that accuracy preservation and efficiency can be achieved simultaneously.
    
    % \item We demonstrate, through extensive experiments on various reasoning benchmarks, that InferSpec achieves higher accuracy and lower latency than existing speculative decoding methods, establishing it as both effective and efficient for real-world LLM deployment.  
    % \item InferSpec improves accuracy by up to 3.6\% while reducing latency by $\sim$11\% compared to state-of-the-art methods, demonstrating that accuracy preservation and efficiency can be achieved simultaneously.

% \begin{figure}[!t]
%     \centering
%     \includegraphics[width=0.8\linewidth]{InferSpec.png}
%     \caption{\ours{}: High level overview of our proposed model}
%     \label{fig:our_model}
% \end{figure}

\section{Related Work}

\label{sec:related_work}

\textbf{Speculative Decoding.} Speculative decoding accelerates inference by letting a lightweight draft model propose tokens that a larger target model verifies in parallel \cite{leviathan2023fast,li2024eagle,chen2024cascade,chen2023accelerating,zhang2024draft,stern2018blockwise,xia2024unlocking,sun2024triforce}. Variants include tree-based speculation \cite{chen2024sequoia,sun2023spectr,fu2024break,miao2024specinfer} to increase acceptance, self-speculative decoding that leverages parts of the base model \cite{zhang2024draft,elhoushi2024layerskip}, and CTC-based drafting \cite{wen2024speculative} to improve sequence quality. Methods like LayerSkip \cite{elhoushi2024layerskip} and Draft-on-the-Fly \cite{metel2024draft} further explore adaptive or early-exit strategies. SpecReason \cite{pan2025specreason} performs speculative reasoning with the target model as a critic that scores semantic utility. \ours{}, in contrast, combines multi-sample self-consistency with an ensemble verifier, enabling stronger filtering of plausible-but-ungrounded steps. RSD \cite{rsd_paper_2025} incorporates process reward models (PRMs) to guide speculative reasoning at the step level. \ours{} differs by keeping the standard draft-target pipeline but replacing external verifiers with lightweight, model-internal signals for step-level evaluation.

\textbf{Reward Models on Reasoning.}
Reward models are used to provide feedback for choosing the correct reasoning path \cite{zhou2025q,wang2024math,chen2024alphamath}. Outcome reward models (ORMs) \cite{dong2024rlhf,yu2024ovm} score final answers, while process reward models (PRMs) \cite{lightman2023let} assess intermediate steps. Advances in reward models have increased focus on scaling test-time compute \cite{snell2024scaling}, enabling strategies like Best-of-N sampling \cite{dong2023raft,cobbe2021training,brown2024large}, tree search \cite{yao2023tree,qimutual,chen2024alphamath}, and reward-guided inference such as RSD \cite{rsd_paper_2025} or SPECS \cite{cemri2025specs}. These improve reasoning quality but add latency and reliance on external verifiers. In contrast, \ours{} leverages an ensemble of internal confidence and grounding signals, avoiding external PRMs while improving multi-step reasoning accuracy.

\begin{figure*}[!t]
    \centering
    \includegraphics[width=0.76\textwidth]{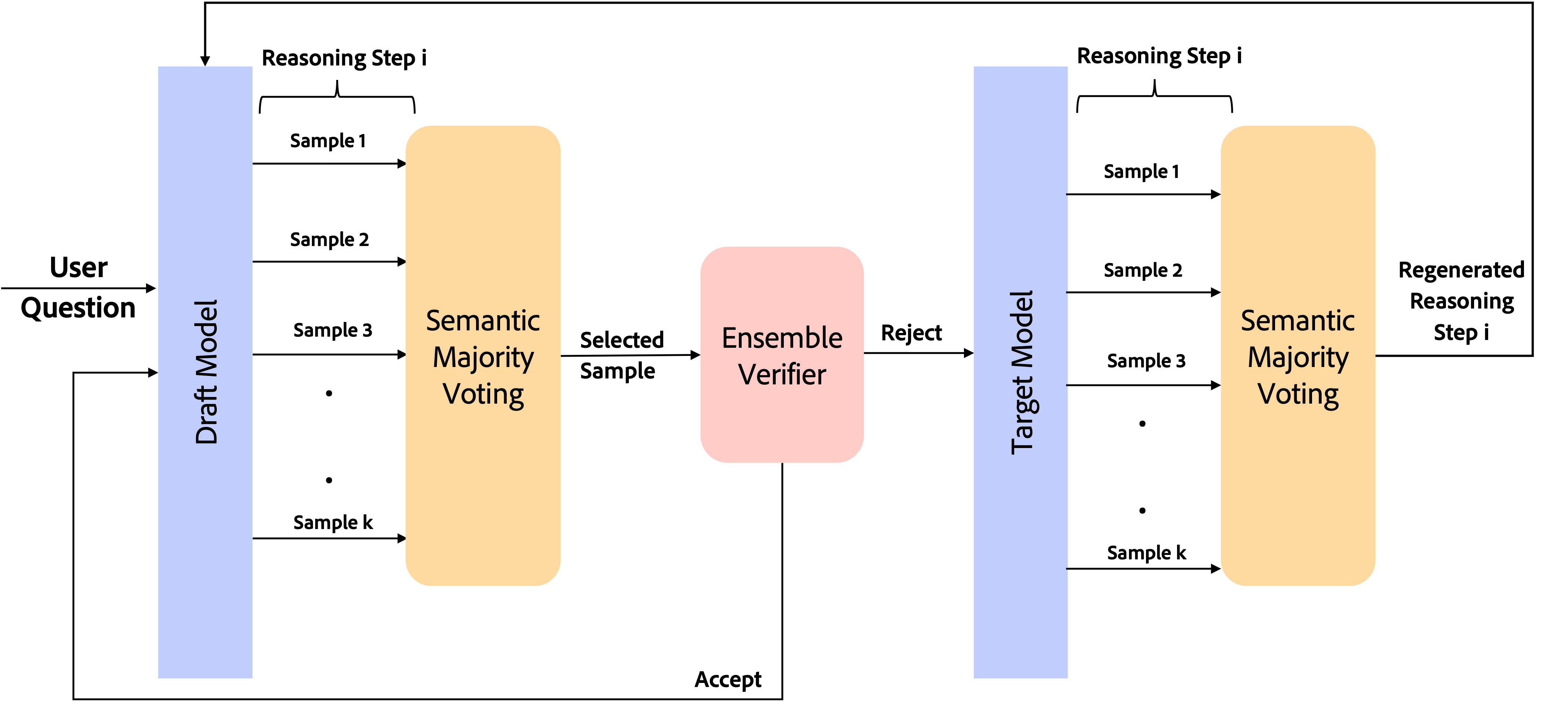}
    \caption{Architectural overview of the \ours{} framework}
    \label{fig:our_model}
\end{figure*}

\section{Our Proposed Approach}

% Large language models (LLMs) demonstrate strong reasoning capabilities, but their deployment is often constrained by high inference costs. Speculative decoding has emerged as a promising strategy to accelerate generation by allowing a lightweight draft model to propose candidate outputs, which are selectively accepted or rejected by a stronger target model. However, ensuring both efficiency and faithfulness in reasoning remains a central challenge, particularly in multi-step tasks such as mathematical problem solving, where early errors can propagate and compromise the final solution. To address this, we 

In this section, we present \ours{}, a novel speculative decoding framework leveraging inference-time compute. It employs an ensemble verifier combining attention-based grounding (Section \ref{sec:attention}) and probability-based signals (Section \ref{sec:logprob}). This formulation enables efficient reasoning without reliance on external verifiers while maintaining interpretability and robustness. 
We then describe how our approach combines the inference time compute with the ensemble-guided acceptance criteria (Section \ref{sec:ensemble_verifier})). Figure \ref{fig:our_model} outlines the high-level architectural overview of our proposed framework.

\subsection{Attention-Based Grounding Verification}

\label{sec:attention}

We introduce Attention-Based Grounding Verification (ABGV) as a mechanism to assess whether each output token (i.e. full reasoning step in our scenario) generated by a language model is sufficiently grounded in the input context or the previously generated steps. Unlike approaches that rely on external verifiers or auxiliary models, ABGV directly leverages the internal attention matrices of the model itself, enabling efficient and scalable verification. The key intuition is that a correctly grounded output should exhibit strong attention alignment with the most relevant input tokens or validated prior steps, thereby reflecting faithful attribution rather than spurious correlations.

Let an input prompt be denoted as $x$, and the language model generates an output sequence $y_i = (y_{i,1}, y_{i,2}, \dots, y_{i,T})$ at each step $i$. At each step, the model produces multilayer multihead attention matrices for each layer $l$ and head $h$:
$
A^{(l,h)} \in \mathbb{R}^{(t_{\text{input}} + t_{\text{output}}) \times (t_{\text{input}} + t_{\text{output}})}
$.

To compute cumulative attribution from input tokens to an output token, we use the well-known {\em attention rollout} mechanism, which recursively multiplies attention matrices across layers to show the total influence of each input token on the final output.
Formally, let $A^{(l)}$ denote the attention matrix averaged over the heads in layer $l$. The rollout matrix $R$ is computed as:
$
R = A^{(L)} A^{(L-1)} \cdots A^{(1)}
$.

For each output token $y_{i,t}$, the distribution over the input tokens is given by the row $R_{y_{i,t}}$ of the rollout matrix, normalized to sum to 1.
Let $\mathcal{I}$ denote the set of input tokens (including prior reasoning steps). 
The grounding score for token $y_{i,t}$ is defined as:
\[
G(y_{i,t}) = \sum_{j \in \mathcal{I}} R_{y_{i,t}}[j]
\]
Here, $R_{y_{i,t}}[j]$ denotes the attention weight of the token $y_{i,t}$ to the input token $j$. 
A higher grounding score indicates a stronger reliance on the input context. 
We adopt a stricter criterion by taking the minimum token grounding score across the reasoning step $y_i$:
$
G_{\text{min-step}} = \min_{t} G(y_{i,t})
$,
which ensures that every token in the generated reasoning step $y_i$ is sufficiently grounded, thereby preventing ungrounded tokens from being masked by averaging.

\noindent \textbf{Memory-Efficient Design:}
A naive implementation would require storing attention matrices from all layers, which could become memory-intensive for larger models and longer outputs. To ensure practical scalability, ABGV employs two lightweight design choices:

\textbf{- Layer subset:} we store attention matrices from the last $3$ layers, which we find sufficient for grounding quality (Figure \ref{fig:ablation} shows minimal loss in performance).

\textbf{- Head sparsification:} we discard entries below $0.01$ in attention heads, reducing memory footprint with negligible effect on grounding fidelity (see Figure \ref{fig:ablation}).

% Together, these optimizations keep ABGV’s memory overhead low even for longer outputs and large-scale models.

\subsection{Log Probability-Based Verification}

\label{sec:logprob}

We introduce log-probability-based verification (LPBV) as a complementary mechanism to assess the reliability of the full reasoning step generated by a language model. LPBV relies on the model’s own predictive confidence, as reflected in the conditional log probability of the tokens generated. The key intuition is that faithfully generated and reliable output should be assigned a relatively high log probability under the model’s next-token distribution, while unreliable tokens are often associated with predictions of low probability.
For each token $y_{i,t}$, of the reasoning step $y_i$, the model produces a conditional probability given the input $x$ and the prior steps:
$
p(y_{i,t} \mid x, y_{i,<t}),
$
from which we compute the logarithmic probability score:
\[
L(y_{i,t}) = \log p(y_{i,t} \mid x, y_{i,<t})
\]

At each step level, 
% these token-level scores can be aggregated in two ways. A natural aggregation is the average logarithmic probability, which reflects the overall confidence of the step:
% \[L_{\text{avg-step}} = \frac{1}{T} \sum_{t=1}^{T} L(y_{i,t})\]
% Alternatively, 
a stricter criterion is applied by taking the minimum log probability across tokens, ensuring that no token is assigned disproportionately low confidence:
$
L_{\text{min-step}} = \min_{t} L(y_{i,t})
$
.

%%%%%%%%%%%%%%%%%%%%%%%%%%%%%%%%%%%%%%%%%%%%
%%%%%%%%%%%%%%%%%% Ensemble Verifier
%%%%%%%%%%%%%%%%%%%%%%%%%%%%%%%%%%%%%%%%%%%%

\subsection{\ours{}: Verification-Aware Speculative Decoding}
\label{sec:ensemble_verifier}

We propose \ours{}, an ensemble verifier-guided speculative decoding framework that augments speculative decoding with principled verification at the step level. At each reasoning step, 
\ours{} evaluates draft outputs using two lightweight, model-internal signals: 
(i) \textit{Log Probability-Based Verification (LPBV)}, which enforces token-level confidence by 
measuring predictive likelihoods, and (ii) \textit{Attention-Based Grounding Verification (ABGV)}, 
which ensures that every generated token is properly attributed to the input or previously validated 
steps via attention rollout. These complementary criteria are combined into a unified ensemble score 
that carries \textbf{formal guarantees}: only steps that are simultaneously confident and grounded 
are accepted, while uncertain steps trigger recomputation with the target model. In doing so, 
\ours{} mitigates error cascades (common in SD), thus improving reasoning reliability while preserving efficiency.

\begin{algorithm}[!t]
\caption{\ours{}}
\label{algo:inferspec}
\footnotesize
\textbf{Input:} Prompt $x$, draft model $m$, target model $M$, log prob function $L(\cdot)$, grounding score function $G(\cdot)$, 
log prob range $[\ell_{\min}, \ell_{\max}]$, grounding range $[g_{\min}, g_{\max}]$, 
weight $\beta$, acceptance threshold $\tau$, EOS token $s$, max length $N$, samples per step $k$

\textbf{Output:} Response $y_{1:i}$

Initialize $y_{1:0} \gets ""$

\For{$i \gets 1$ \KwTo $N-1$}{
    Sample $k$ draft candidates $\{\hat{y}_i^{(1)}, \dots, \hat{y}_i^{(k)}\} \gets m(x, y_{1:i-1})$
    
    Select draft step $\hat{y}_i^{j*} \gets \text{Self-Consistency Selector}(\{\hat{y}_i^{(j)}\}_{j=1}^k)$
    
    Compute min log prob $\ell_i \gets L(\hat{y}_i^{j*})$
    
    Compute min grounding score $g_i \gets G(\hat{y}_i^{j*})$ 
    
    Normalize:
    $\tilde{\ell}_i \hspace{-0.12cm} = \hspace{-0.12cm} \dfrac{\ell_i - \ell_{\min}}{\ell_{\max} - \ell_{\min}}$, \hspace{-0.1cm}
    $\tilde{g}_i \hspace{-0.12cm} = \hspace{-0.12cm} \dfrac{g_i - g_{\min}}{g_{\max} - g_{\min}}$
    
    Compute ensemble verifier score:
    $r_i \gets \beta \cdot \tilde{\ell}_i + (1-\beta) \cdot \tilde{g}_i$
    
    \eIf{$r_i \geq \tau$}{
        Accept draft step $y_i \gets \hat{y}_i^{j*}$
    }{
        Sample $k$ target candidates $\{y_i^{(1)}, \dots, y_i^{(k)}\} \gets M(x, y_{1:i-1})$
        
        Select target step $y_i^{j*} \gets \text{Self-Consistency Selector}(\{y_i^{(j)}\}_{j=1}^k)$

        $y_i \gets y_i^{j*}$
    }
    
    \If{$s \in y_i$}{
        \textbf{break}
    }
}
\end{algorithm}

% We design \ours{}, an ensemble verifier-guided speculative decoding algorithm that extends the speculative decoding framework by incorporating an ensemble of verification signals to determine whether the reasoning steps generated by the draft should be accepted or regenerated by the target model. Specifically, \ours{} uses two complementary criteria: (i) \textit{Log Probability-Based Verification (LPBV)}, which captures the model’s predictive confidence, and (ii) \textit{Attention-Based Grounding Verification (ABGV)}, which evaluates whether generated tokens are sufficiently grounded in the input context and previously validated steps. By integrating these two distinct signals, \ours{} ensures that accepted steps are not only probable under the model distribution but also contextually grounded, thus improving reliability without sacrificing efficiency.  
At step $i$, \ours{} proceeds as follows:
\vspace{0.2cm}\\
% \begin{enumerate}
\textbf{A. Generate Draft Step:} The draft model $m$, samples $k$ candidate steps $\{\hat{y}_i^{(1)}, \dots, \hat{y}_i^{(k)}\}$ conditioned on the input prompt and previously accepted steps. To identify the most consistent candidate from these $k$ possibilities, we propose the {\em self-consistency selector} (see Section \ref{sec:self_consistency_selector}), which selects the step $\hat{y}_i^{j^*}$ that is maximally consistent with the other $k-1$ candidates.
\vspace{0.2cm}\\
\textbf{B. Compute Verification Scores:} For the selected step, the ensemble verifier computes both the logarithmic probability-based score $L(\hat{y}_i^{j^*})$ and the grounding score $G(\hat{y}_i^{j^*})$. Before aggregation, both scores are scaled to a comparable range using Min-Max normalization. 
\vspace{0.2cm}\\
\textbf{C. Apply Acceptance Criterion:} The ensemble verifier combines normalized scores through a weighted aggregation to determine acceptance. If the criterion is satisfied, $\hat{y}_i^{j^*}$ is accepted; otherwise, the target model $M$ is invoked to sample $k$ candidate steps $\{y_i^{(1)}, \dots, y_i^{(k)}\}$ to reduce stochastic variance and improve reliability. Since even the target model may occasionally produce inconsistent reasoning. 
% This is done to reduce stochastic variance and improve step-wise reliability. Since even the target model can occasionally produce locally inconsistent or semantically divergent reasoning, relying on a single sample may propagate errors.
{\em Self-consistency selector} is again applied to select the most consistent step $y_i^{j^*}$.
\vspace{0.2cm}\\
\textbf{D. Repeat Until Termination:} This process continues until the model generates an end-of-sequence (EOS) token or the sequence reaches the maximum length $N$. 
\vspace{0.2cm}\\ 
% \noindent
Algorithm \ref{algo:inferspec} outlines the key steps involved in the proposed approach \ours{}. Analysis of the computational complexity of \ours{} is provided in Appendix \ref{complexity_analysis}.

% \begin{algorithm}[H]
% \caption{Speculative Decoding with Ensemble Verifier}
% \KwIn{Prompt $x$, draft model $m$, target model $M$, 
% min log prob function $\ell(\cdot)$, grounding score function $g(\cdot)$, 
% log prob range $[\ell_{\min}, \ell_{\max}]$, grounding range $[g_{\min}, g_{\max}]$, 
% weight $\beta$, acceptance threshold $\tau$, EOS token $s$, max length $N$}
% \KwOut{Response $y_{1:i}$}

% Initialize $y_{1:0} \gets ""$

% \For{$i \gets 1$ \KwTo $N-1$}{
%     Generate draft step $\hat{y}_i \gets m(x, y_{1:i-1})$
    
%     Compute min log prob $\ell_i \gets \ell(\hat{y}_i)$
    
%     Compute min grounding score $g_i \gets g(\hat{y}_i)$
    
%     Normalize:
%     $\tilde{\ell}_i = \dfrac{\ell_i - \ell_{\min}}{\ell_{\max} - \ell_{\min}}$, 
%     $\tilde{g}_i = \dfrac{g_i - g_{\min}}{g_{\max} - g_{\min}}$
    
%     Compute combined verifier score:
%     $r_i \gets \beta \cdot \tilde{\ell}_i  + (1-\beta) \cdot \tilde{g}_i$
    
%     \eIf{$r_i \geq \tau$}{
%         Accept draft step $y_i \gets \hat{y}_i$
%     }{
%         Generate target step $y_i \gets M(x, y_{1:i-1})$
%     }
    
%     \If{$s \in y_i$}{
%         \textbf{break}
%     }
% }
% \end{algorithm}

% Please refer to Appendix \ref{complexity_analysis} for an analysis on the computational complexity of \ours{}.

\begin{algorithm}[!t]
\caption{Self-Consistency Selector}
\label{algo:self_consistency_selector}
\footnotesize
\KwIn{Set of $k$ candidates $\{y^{(1)}, \dots, y^{(k)}\}$, sentence transformer model $\mathcal{E}$}
\KwOut{Index $j^\ast$ of the selected candidate}

Compute embeddings for all candidates: 
$e^{(j)} \gets \mathcal{E}({y^{(j)}})$ for $j = 1 \dots k$

Normalize embeddings so that $\|e^{(j)}\|_2 = 1$

Compute pairwise similarity matrix:
$S_{ij} \gets \langle e^{(i)}, e^{(j)} \rangle$ for $i,j = 1 \dots k$

Apply row-wise softmax:
$\tilde{S}_{ij} = \dfrac{\exp(S_{ij})}{\sum_{l=1}^k \exp(S_{il})}$

Extract diagonal scores:
$d_j \gets \tilde{S}_{jj}$ for $j = 1 \dots k$

% Select candidate with minimum self-alignment:
$j^\ast \gets \arg \min_j d_j$ \tcp*[r]{Select most semantically consistent candidate}
\end{algorithm}

%%%%%%%%%%%%%%%%%%%%%%%%%%%%%%%%%%%%%%%%%%%%%%
%%%%%%%%%%%%%%% Self-Consistency Selector
%%%%%%%%%%%%%%%%%%%%%%%%%%%%%%%%%%%%%%%%%%%%%%
\subsection{Self-Consistency Selector to Identify Consistent Candidate}
\label{sec:self_consistency_selector}

To identify the most consistent reasoning step among a set of $k$ sampled candidates (either by draft or target), we propose the {\em self-consistency selector}, based on \cite{dcscore_icml_2025}.
% a semantic consistency-based selection criterion. 
The underlying intuition is that a consistent candidate should exhibit strong agreement with the other candidates, rather than being an outlier. Formally, each candidate $y^{(j)}$ is encoded in a normalized embedding $e^{(j)}$ using a pre-trained sentence transformer $\mathcal{E}$. While embeddings could, in principle, be obtained from the model’s final hidden layer, vLLM does not expose intermediate states, so we compute them externally. The (cosine) similarities are then calculated in pairs between the candidates, yielding a similarity matrix $S \in \mathbb{R}^{k \times k}$, which is further normalized row-wise using softmax to obtain $\tilde{S}$. For each candidate $y^{(j)}$, we calculate its self-alignment score $d_j = \tilde{S}_{jj}$, which measures the degree to which the candidate aligns with itself relative to the others. Candidates that are semantically consistent with the rest of the set distribute their similarity mass across multiple peers, producing a low $d_j$, while outliers or less consistent candidates concentrate the similarity primarily on themselves, resulting in a high $d_j$. Thus, candidates with lower $d_j$ are more representative of the set, and our approach selects the candidate with the minimum self-alignment score: $j^\ast \gets \arg \min_j d_j$. Algorithm \ref{algo:self_consistency_selector} describes our novel self-consistency selector.

\subsection{Formal Guarantees}

We now present formal guarantees for the proposed \ours{} algorithm. 

\begin{lemma}[Soundness Guarantee]
\label{lem:soundness}
Let $\mathcal{C}$ denote the set of correct reasoning steps, $\tilde{\ell}_i$ be the logarithmic probability signal, and $\tilde{g}_i$ be the attention-grounding signal. For any $\alpha \in [0,1]$, $\epsilon_\ell \in [0,1]$ and $\epsilon_g \in [0,1]$, assume that $\Pr\!\left[\tilde{\ell}_i \geq \alpha \;\middle|\; y_i \in \mathcal{C}\right] \geq 1 - \epsilon_\ell,
\quad
\Pr\!\left[\tilde{g}_i \geq \alpha \;\middle|\; y_i \in \mathcal{C}\right] \geq 1 - \epsilon_g.$
% \[
% \Pr\!\left[\tilde{\ell}_i \geq \alpha \;\middle|\; y_i \in \mathcal{C}\right] \geq 1 - \epsilon_\ell,
% \quad
% \Pr\!\left[\tilde{g}_i \geq \alpha \;\middle|\; y_i \in \mathcal{C}\right] \geq 1 - \epsilon_g.
% \]
Then, 
$
\Pr\!\left[ V(y_i) = \texttt{accept} \;\middle|\; y_i \in \mathcal{C} \right] \geq 1 - (\epsilon_\ell + \epsilon_g).
$
\end{lemma}

\begin{proof}
Both $\tilde{\ell}_i$ and $\tilde{g}_i$ 
independently provide high probability acceptance for correct steps. 
Since the ensemble score satisfies 
$r_i \geq \min(\tilde{\ell}_i, \tilde{g}_i)$, 
the probability of rejection is bounded by the union of individual error events. 
Then, the total error probability is at most $\epsilon_\ell + \epsilon_g$ and thus the lemma follows.
\end{proof}

\begin{lemma}[Efficiency Guarantee]
\label{lem:efficiency}
Let $\pi_i = \Pr[V(y_i) = \texttt{accept}]$. 
Then the expected no. of target calls $C_T$, is
$
\mathbb{E}[C_T] = \sum_{i=1}^T (1 - \pi_i)
$. If $\pi_i \geq \pi_{\min} \;$ for all $i$, then
$
\mathbb{E}[C_T] \leq T \cdot (1 - \pi_{\min}).
$
\end{lemma}

\begin{proof}
At each step $i$, a target call is required if and only if $V(y_i) = \texttt{reject}$. 
Thus, the expectation is $\sum_i (1 - \pi_i)$. 
If $\pi_i \geq \pi_{\min}$, the sum is bounded by $T(1 - \pi_{\min})$. 
This formalizes that higher acceptance rates directly reduce expected target calls. 
\end{proof}

\begin{theorem}[Accuracy--Efficiency Trade-off]
\label{thm:tradeoff}
Suppose correct steps satisfy Lemma~\ref{lem:soundness} and the incorrect steps are rejected with probability at least $1 - \delta$. 
Then, for any sequence of length $T$,
$
\Pr\!\left[\text{all accepted steps are correct}\right] 
\geq (1 - \epsilon_\ell - \epsilon_g)^T \cdot (1 - \delta)^{C_T}.
$
\end{theorem}

\begin{proof}
By Lemma~\ref{lem:soundness}, the probability of accepting only the correct steps
is at least $(1 - \epsilon_\ell - \epsilon_g)^T$. 
By assumption, incorrect steps are rejected with probability at least $1 - \delta$, 
and there are $C_T$ target calls. 
Thus, the lower bound of the joint probability is
$(1 - \epsilon_\ell - \epsilon_g)^T (1 - \delta)^{C_T}$. 
\end{proof}

These results show that \ours{} provides multiplicative accuracy guarantees while bounding the expected number of target calls.

\begin{table*}[!t]
    \footnotesize
    \small
    \centering
\caption{Accuracy on reasoning benchmarks.}
\label{tab:acc_reasoning}
\resizebox{\textwidth}{!}{
\begin{tabular}{lccccccc}
\toprule
\textbf{Method} & \textbf{Target Model} & \textbf{Draft Model}   &  \textbf{MATH500} & \textbf{GSM8K} & \makecell{\textbf{Gaokao} \\ \textbf{2023 En}} & \makecell{\textbf{Olympiad} \\ \textbf{Bench}} \\

\midrule

\multicolumn{7}{c}{\cellcolor{orange!15} \textbf{Math Model, Draft and Target: Qwen2.5-Math-Instruct}} \\
\midrule
Target Model   & 7B & - &  83.0 & 94.7 & 66.8 & 40.6   \\
Target-only Majority &  7B &  - &  84.9	&  95.2	& 68.8	&  41.0 \\
Draft-only Majority   & - & 1.5B  & 79.0 & 88.9 & 67.9 & 40.9   \\
Best-of-$N$   & - & 1.5B   & 82.2 & 93.3 & 67.4 & 40.7   \\
SD   & 7B & 1.5B  & 82.4 & 94.2 &  66.3 & 39.4  \\
RSD  & 7B  & 1.5B  & 82.4 & 94.4 & 68.5 & 39.6 \\
RSD Majority  & 7B  & 1.5B  & 78.0 & 88.7 & 64.9 & 38.7 \\

 SC + LPBV &  7B  &  1.5B	&  83.2	&  94.5	&  67.5	& 39.7 \\ 

\oursgreedy{} & 7B & 1.5B  & 83.6 & 95.6 & 68.8 & 40.7 \\
\rowcolor{yellow!10} \ours{} & 7B & 1.5B  & \textbf{85.4} & \textbf{95.8} & \textbf{69.4} & \textbf{41.2} \\

\midrule
\multicolumn{7}{c}{\cellcolor{orange!15} \textbf{General Model, Draft and Target: Qwen2.5-Instruct}} \\
\midrule
Target Model  & 7B & -    & 74.8 & 91.7 & 64.9 & 38.8   \\
Draft-only Majority   & - & 1.5B  & 66.4 & 82.1 & 56.9 & 28.7  \\
Best-of-$N$   & - & 1.5B  & 73.4 & 89.7 & 60.5 & 32.7   \\
SD   & 7B & 1.5B  & 74.8 & 91.6 & 63.1 & 37.1   \\
RSD   & 7B & 1.5B  & 71.4 & 90.1 & 60.5 & 37.6   \\
RSD Majority  & 7B & 1.5B  & 60.6 & 77.0 & 55.3 & 31.7 \\

\oursgreedy{} & 7B & 1.5B  & 74.9 & 92.0 & 65.5 & 37.8 \\
\rowcolor{yellow!10} \ours{} & 7B & 1.5B  & \textbf{77.0} & \textbf{93.0} & \textbf{66.0} & \textbf{40.3} \\

\midrule

\multicolumn{7}{c}{\cellcolor{orange!15} \textbf{General Model, Draft: Llama-3.2-Instruct and Target: Llama-3.1-Instruct}} \\
\midrule
Target Model   & 8B & -   & 48.2 & 83.9 & 40.8 & 14.5   \\
Draft-only Majority  & -  & 1B  & 38.0 & 60.2 & 32.2 & 9.5   \\
Best-of-$N$   & - & 1B   & 48.6 & 74.8 & 40.7 & 14.4   \\
SD   & 8B  & 1B & 47.0 & 83.4 & 40.1 & 16.1   \\
RSD   & 8B & 1B  & 50.0 & 83.9 & 41.8 & 15.7   \\
RSD Majority & 8B & 1B  & 36.6 & 61.9 & 30.6 & 12.3 \\

\oursgreedy{} & 8B & 1B  & 50.0 & 84.5 & 41.9 & 16.9 \\
\rowcolor{yellow!10} \ours{} & 8B & 1B  & \textbf{51.6} & \textbf{85.1} & \textbf{43.9} & \textbf{17.2} \\

\bottomrule
\end{tabular}}
\end{table*}

\section{Experiments}

\label{sec:results}

% We aim to address the following research questions:\\
% \noindent\textbf{RQ I.} Can ensemble-verifier-guided speculative decoding improve accuracy on challenging reasoning benchmarks compared to the state-of-the-art methods?\\
% \noindent\textbf{RQ II.} How does varying the number of samples per reasoning step impact accuracy?\\
% \noindent\textbf{RQ III.} Can we achieve lower runtime than RSD while maintaining comparable or better accuracy?

We address the following research questions:\\
\noindent\textbf{RQ I.} Does \ours{}
provide measurable accuracy improvements on multi-step reasoning benchmarks compared 
to state-of-the-art methods, while mitigating error cascades?\\
\noindent\textbf{RQ II.} How does the number of sampled candidates per reasoning step 
influence both the reliability and stability of \ours{} under the ensemble verification criterion?\\
\noindent\textbf{RQ III.} Can \ours{} reduce inference latency compared to RSD while preserving, or even enhancing, accuracy guarantees?

 \subsection{Experimental Setup}
  \label{sec:experiments_sec_1}

\textbf{Datasets and Metrics:} We conduct extensive experiments on datasets that require complex reasoning, including MATH500~\cite{hendrycks2measuring}, GSM8K~\cite{cobbe2021training}, GaoKao-2023-En~\cite{liao2024mario}, and OlympiadBench~\cite{he2024olympiadbench}. For evaluation, we adopt the official metrics, i.e., \textit{exact match} (EM). Detailed descriptions of the datasets can be found in Appendix~\ref{sec:datasets}.

\textbf{Models:} To evaluate \ours{}'s effectiveness, we consider both general-purpose and math-focused LLMs, namely Qwen-2.5-Math \cite{yang2024qwen2}, Qwen-2.5 \cite{qwen2025qwen25technicalreport}, and Llama-3 \cite{dubey2024llama}. For RSD, we adopt Skywork-o1-OpenPRM \cite{skywork} as the PRM. 
% The PRM produces a score between 0 and 1, with higher values indicating better quality.

\textbf{Baselines:} We evaluate \ours{} against four categories of baselines:
(1) \textit{Target model only}: the target model is used independently, which generally incurs a higher computational cost compared to \ours{}. (2) \textit{Draft model with or without PRM:} This group covers inference time compute techniques that aim to maximize the performance of the draft model. Specifically, we evaluate majority voting and Best-of-N (BoN) \cite{brown2024large,cobbe2021training}, where BoN selects the highest scoring response (last step) among N candidates using a PRM; beam search \cite{chen2024alphamath}, which employs a PRM to identify the optimal decoding trajectory; and we process Best-of-N, which samples N candidate steps and chooses the one with the highest reward. (3) \textit{Speculative decoding (SD):} We also include speculative decoding with 7 speculative tokens for accelerating inference \cite{leviathan2023fast}. (4) \textit{RSD:} \cite{rsd_paper_2025} leverages a PRM to score intermediate steps and adaptively determine when to call target model.

\textbf{Setting:} We perform all experiments on NVIDIA A100 GPUs with vLLM as the serving backend. We define a reasoning step as a generation terminated by \textbackslash n\textbackslash n. For generating multiple samples, we set $temperature = 0.7$, $top\_p = 0.8$, and $n = 16$. \ours{} refers to this multi-sample setting combined with our self-consistency selector, which chooses the most representative reasoning step. RSD Majority likewise employs this multi-sample setting, with each step scored by a PRM and the highest-scoring candidate selected.
In the greedy setting where $temperature = 0$, $top\_p = 1$, and $n = 1$, we refer our approach as \oursgreedy{}.
% For greedy decoding, we use temperature = 0, top\_p = 1, and n = 1. \oursgreedy{} refer to this setting. 
% In the case of RSD, each step is evaluated using a Process Reward Model (PRM), while for \ours{}, the step is evaluated using the proposed ensemble verifier. 
For both RSD and \ours{}, we set the threshold parameter to $\tau$ = 0.7 and we set $\beta$ = 0.3 for \ours{}. For details about hyperparameters refer to Appendix \ref{sec:hyperparameter}. Unless stated otherwise, all models used are Qwen-2.5-Math-Instruct. 
% For details about hyperparameters refer to Appendix \ref{sec:hyperparameter}.

\begin{table}[!t]
    \footnotesize
    \small

    \centering
\caption{Comparison with search-based methods on Qwen2.5-Math-Instruct. Beam Search and Process Best-of-$N$ use a 1.5B base model and a 1.5B PRM.}
\label{tab:acc_search}
\resizebox{\linewidth}{!}{
\begin{tabular}{lcccc}
\toprule
\textbf{Method} & \textbf{Setting} & \textbf{MATH500} & \textbf{GSM8K}  \\
\midrule
Draft Model (1.5B)      & -      & 73.8 & 85.0   \\
Process Best-of-$N$      & $N=8$  & 75.8 & 87.8   \\
Process Best-of-$N$      & $N=16$ & 76.0 & 87.9   \\
Beam Search              & bs=4   & 78.2 & 88.9  \\
Beam Search              & bs=8   & 78.2 & 88.4   \\
RSD (1.5B/7B/1.5B) & - & 82.4 & 94.4  \\
\oursgreedy{} & - & 83.6 & 95.6 \\
\rowcolor{blue!8} \ours{} & maj@16 & \textbf{85.4} & \textbf{95.8} \\
\bottomrule
\end{tabular}}
\end{table}

\begin{figure}[!t]
\small
\centering
\begin{subfigure}{.8\linewidth}
\begin{tikzpicture}
\begin{axis}[
    ylabel style = {font= \tiny},
    xlabel style = {font= \tiny},
    xlabel=No. of samples,
    xticklabel style={font=\tiny},
    yticklabel style={font=\tiny},
    ylabel=Accuracy on MATH500,
    height=5.0cm,
    width=6.2cm,
    legend style={
                    font= \tiny,
                },
    legend cell align={left},
    legend pos = {north west},
    xmin=0, xmax=17,
    ymin=75, ymax=90]
\addplot[smooth,mark=*,blue] plot coordinates {
    (1,83.6)
    (4,84.2)
    (8,85.0)
    (16,85.4)};
\addlegendentry{SpecGuard}
\addplot[smooth,mark=*,red] plot coordinates {
    (1,82.4)
    (4,81.2)
    (8,80.6)
    (16,78.0)};
\addlegendentry{RSD Majority}
\end{axis}
    \end{tikzpicture}
\begin{tikzpicture}
\edef\ours{"85.4","95.8","69.4","41.2"}
\edef\rsd{"78.0", "88.7", "64.9", "38.7"}

    \begin{axis}[
            ybar=5pt,
            width=1\textwidth,
            bar width=0.25,
            height=5cm,
            width=6.2cm,
            every axis plot/.append style={fill},
            grid=major,
            xtick={1, 2, 3, 4, 5},
            xticklabels={MATH500, GSM8K, GaoKao, Olympiad},
            xlabel={},
            ylabel style = {font=\tiny},
        yticklabel style = {font= \tiny,xshift=0.05ex},
        xticklabel style ={font=\tiny,yshift=0.5ex},
            ylabel={Time (in minutes)},
            enlarge x limits=0.15,
            ymin=0,
            ymax=90,
            legend style ={font=\tiny,yshift=0.5ex},
            area legend,
            nodes near coords style={font=\tiny,align=center,text width=2em},
            legend entries={SpecGuard, RSD Majority},
            legend cell align={left},
            legend pos=north west,
            legend style={/tikz/every even column/.append style={column sep=0.5cm}},
        ]
        \addplot+[
            ybar,
            plotColor2*,
            nodes near coords=\pgfmathsetmacro{\mystring}{{\ours}[\coordindex]}\textbf{\mystring},
            nodes near coords align={vertical},
            draw=black,
            postaction={
                    pattern=north west lines
                },
        ] plot coordinates {
                (1,41)
                (2,34)
                (3,35)
                (4,73)
            };

        \addplot+[
            ybar,
            plotColor1*,
            draw=black,
            nodes near coords=\pgfmathsetmacro{\mystring}{{\rsd}[\coordindex]}\textbf{\mystring},
    nodes near coords align={vertical},
            postaction={
                    pattern=north west lines
                },
        ] plot coordinates {
                (1,46)
                (2,41)
                (3,43)
                (4,80)
            };

    \end{axis}

\end{tikzpicture}
\end{subfigure}
\caption{(\textit{Top}) Varying number of samples. (\textit{Bottom}) Runtime comparison (y-axis) RSD Majority vs \ours{} with corresponding \textbf{accuracy indicated on top of bars}.}
\label{fig:varying_samples}
\end{figure}
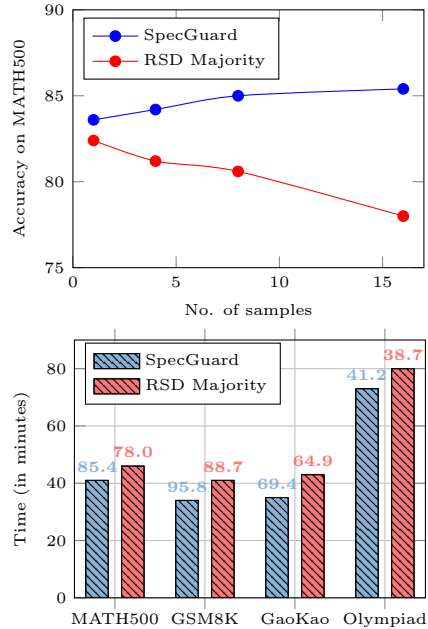

\subsection{Performance Comparison}
To address \textbf{RQI}, we evaluate \ours{} on a broad set of challenging reasoning benchmarks, in Table~\ref{tab:acc_reasoning}, and make the following observations:
(1) Inference-time compute strategies such as majority voting and Best-of-N, which rely on extensive draft sampling, typically underperform compared to the accuracy of a single target model.
% (1) Inference-time compute approaches such as majority voting and Best-of-N, which depend on extensive sampling with a draft model, generally fall short of the performance achieved by a single target model. 
This shows the critical role of larger models in reasoning tasks, as their capabilities cannot be readily matched by smaller models even with increased computation.
(2) While target-only majority voting may match \ours{} in accuracy, it incurs substantially higher computational cost as every reasoning step must be sampled multiple times from the target, contrary to our objective of reducing target calls. 
% In contrast, \ours{} minimizes target queries via speculative reasoning, achieving a superior efficiency-performance trade-off, which is the primary motivation of our work. 
(3) Although speculative decoding (SD) is theoretically unbiased, matching target model accuracy, it often performs worse in practice, primarily due to floating-point errors \cite{chen2023accelerating}. Furthermore, when the draft model surpasses the target model, the strict unbiased nature of SD can actually degrade performance relative to the draft. 
(4) Reward-guided speculative decoding (RSD) addresses this by using a process reward model (PRM) to evaluate draft reasoning steps, but reliance on an external verifier adds latency and computational overhead.
(5) \ours{} uses lightweight model-internal verifiers instead of PRMs to evaluate draft steps, and consistently outperforms both the target model and RSD across all benchmarks, demonstrating strong efficiency and performance. While LPBV captures confidence, it lacks grounding, allowing confident yet ungrounded steps to pass. In contrast, \ours{} achieves higher accuracy, showing ABGV is crucial for rejecting plausible but ungrounded steps. Appendix~\ref{sec:qualitative} presents \textbf{qualitative analysis} showing that PRM often assigns high scores to incorrect draft steps, underscoring the \textit{need for stronger verification that ensures both step-wise soundness and final-answer correctness.}
% In Appendix~\ref{sec:qualitative}, we show \textbf{qualitative analysis} of reasoning steps scored by PRM. Even though all draft-generated steps receive high acceptance scores from the PRM, the final answer is still incorrect, highlighting the \textit{need for stronger verification methods that ensure both step-wise soundness and final-answer correctness.} 
% While LPBV captures confidence, it lacks grounding, so confident but ungrounded steps frequently slip through. In contrast, \ours{} achieves higher accuracy, demonstrating that ABGV is essential for rejecting ungrounded steps that appear locally plausible. 

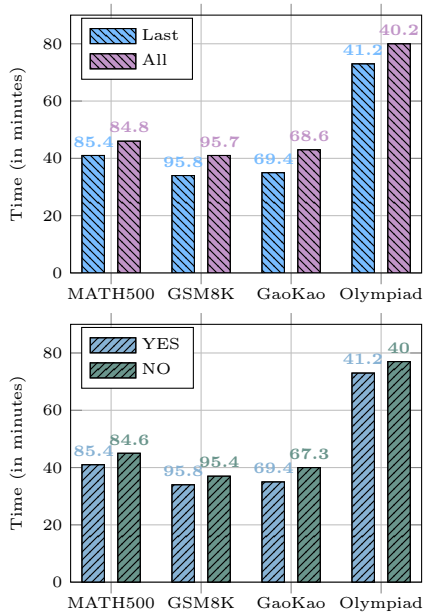
\begin{figure}[!t]
\small
\centering
\begin{subfigure}{0.8\linewidth}
\begin{tikzpicture}
\edef\last{"85.4","95.8","69.4","41.2"}
\edef\all{"84.8","95.7","68.6","40.2"}

    \begin{axis}[
            ybar=5pt,
            width=1\textwidth,
            bar width=0.25,
            height=5cm,
            width=6.2cm,
            every axis plot/.append style={fill},
            grid=major,
            xtick={1, 2, 3, 4, 5},
            xticklabels={MATH500, GSM8K, GaoKao, Olympiad},
            xlabel={},
            ylabel style = {font=\tiny},
        yticklabel style = {font= \tiny,xshift=0.05ex},
        xticklabel style ={font=\tiny,yshift=0.5ex},
            ylabel={Time (in minutes)},
            enlarge x limits=0.15,
            ymin=0,
            ymax=90,
            legend style ={font=\tiny,yshift=0.5ex},
            area legend,
            nodes near coords style={font=\tiny,align=center,text width=2em},
            legend entries={Last, All},
            legend cell align={left},
            legend pos=north west,
            legend style={/tikz/every even column/.append style={column sep=0.5cm}},
        ]
        \addplot+[
            ybar,
            plotColor6*,
            nodes near coords=\pgfmathsetmacro{\mystring}{{\last}[\coordindex]}\textbf{\mystring},
            nodes near coords align={vertical},
            draw=black,
            postaction={
                    pattern=north west lines
                },
        ] plot coordinates {
                (1,41)
                (2,34)
                (3,35)
                (4,73)
            };

        \addplot+[
            ybar,
            plotColor4*,
            draw=black,
            nodes near coords=\pgfmathsetmacro{\mystring}{{\all}[\coordindex]}\textbf{\mystring},
    nodes near coords align={vertical},
            postaction={
                    pattern=north west lines
                },
        ] plot coordinates {
                (1,46)
                (2,41)
                (3,43)
                (4,80)
            };

    \end{axis}

\end{tikzpicture}
\begin{tikzpicture}
\edef\withSparsity{"85.4","95.8","69.4","41.2"}
\edef\withoutSparsity{"84.6","95.4","67.3","40"}

    \begin{axis}[
            ybar=5pt,
            width=1\textwidth,
            bar width=0.25,
            height=5cm,
            width=6.2cm,
            every axis plot/.append style={fill},
            grid=major,
            xtick={1, 2, 3, 4, 5},
            xticklabels={MATH500, GSM8K, GaoKao, Olympiad},
            xlabel={},
            ylabel style = {font=\tiny},
        yticklabel style = {font= \tiny,xshift=0.05ex},
        xticklabel style ={font=\tiny,yshift=0.5ex},
            ylabel={Time (in minutes)},
            enlarge x limits=0.15,
            ymin=0,
            ymax=90,
            legend style ={font=\tiny,yshift=0.5ex},
            area legend,
            nodes near coords style={font=\tiny,align=center,text width=2em},
            legend entries={YES, NO},
            legend cell align={left},
            legend pos=north west,
            legend style={/tikz/every even column/.append style={column sep=0.5cm}},
        ]
        \addplot+[
            ybar,
            plotColor2*,
            nodes near coords=\pgfmathsetmacro{\mystring}{{\withSparsity}[\coordindex]}\textbf{\mystring},
            nodes near coords align={vertical},
            draw=black,
            postaction={
                    pattern=north east lines
                },
        ] plot coordinates {
                (1,41)
                (2,34)
                (3,35)
                (4,73)
            };

        \addplot+[
            ybar,
            plotColor7*,
            draw=black,
            nodes near coords=\pgfmathsetmacro{\mystring}{{\withoutSparsity}[\coordindex]}\textbf{\mystring},
    nodes near coords align={vertical},
            postaction={
                    pattern=north east lines
                },
        ] plot coordinates {
                (1,45)
                (2,37)
                (3,40)
                (4,77)
            };

    \end{axis}

\end{tikzpicture}
\end{subfigure}
\caption{Runtime comparison (y-axis) with corresponding \textbf{accuracy indicated on top of bars}. (\textit{Top}) changing layers. (\textit{Bottom}) Sparsity in attention heads by \ours{}.}
\label{fig:ablation}
\end{figure}

\subsection{Comparison with Search-Based Approaches}
We also compare \ours{} with beam search \cite{chen2024alphamath} and process Best-of-N, in Table~\ref{tab:acc_search}. Our method consistently outperforms both baselines. These findings reveal that when reasoning steps become complex, search-based techniques face limitations, as combinatorial growth of candidate solutions makes it difficult to reliably identify optimal paths, resulting in degraded performance. In contrast, \ours{} leverages the power of larger models to generate strong candidate solutions. In addition, the incorporation of an ensemble verifier provides step-level feedback, mitigating the challenges of difficult tasks.
{\em This highlights that moving beyond purely search-based strategies and augmenting larger models with lightweight feedback mechanisms can deliver both higher efficiency and stronger performance, when the search space is vast or the reasoning task is highly complex.}
%This demonstrates that moving beyond purely search-based strategies and instead combining larger models with lightweight feedback mechanisms can yield both greater efficiency and stronger performance in settings where the search space is large or reasoning complexity is high.

\subsection{Sample Size and Time Comparison}

To investigate \textbf{RQ II}, we evaluate \ours{} under varying sample sizes per reasoning step. 
In Figure~\ref{fig:varying_samples}, accuracy steadily improves as more diverse
candidates are explored, with gains saturating at higher counts. In contrast, RSD Majority exhibits
diminishing returns and even degrades performance with larger samples, due to accumulated
noise from PRM. This demonstrate that \ours{} takes advantage of additional candidate generations more effectively.

For \textbf{RQ III}, Figure~\ref{fig:varying_samples} compares the runtime of RSD Majority and \ours{}, 
with the accuracy indicated on top of the bars. \ours{} consistently achieves both higher accuracy and
lower latency. For example, on GSM8K, \ours{} achieves an accuracy of $95.8\%$ in $34$ minutes, compared to the accuracy of the RSD Majority $88.7\%$ in more than $41$ minutes. Together, these results confirm that our ensemble verifier-guided speculative decoding framework improves reasoning reliability while
delivering superior efficiency.

% To investigate \textbf{RQ II}, we evaluate \ours{} with different sample sizes per reasoning step. As shown in Figure~\ref{fig:varying_samples}(a), the accuracy of \ours{} steadily improves, as more diverse reasoning candidates are explored with gains eventually saturating at higher counts. In contrast, RSD shows diminishing returns and even degrades with larger samples, due to noise from the PRM. These results highlight that \ours{} effectively leverages extra candidate generations compared to RSD.

% To address \textbf{RQ III}, Figure~\ref{fig:varying_samples}(b) reports the runtime of RSD and \ours{}, with accuracy values indicated on top of the bars. \ours{} consistently achieves higher accuracy while also running faster. For instance, on GSM8K, our approach reaches 95.8\% accuracy in 34 minutes, compared to RSD’s 88.7\% accuracy in over 41 minutes. These results highlight that our ensemble-verifier-guided speculative decoding not only boosts accuracy but also offers superior efficiency compared to RSD.

%%%%%%%%%%%%%%%%%%%%%%%%%%%%%%%%%%%%
%%%%%% Ablation Studies
%%%%%%%%%%%%%%%%%%%%%%%%%%%%%%%%%%%%
\subsection{Ablation Studies}
We perform ablation studies to analyze key design choices in \ours{}, shown in Figure~\ref{fig:ablation}.\\
\textbf{Changing Layers:} Figure~\ref{fig:ablation} compares using attention from the last three
layers versus aggregating across all layers. Although the latter yields marginal gains on some benchmarks, it adds runtime overhead. Leveraging only the last three layers strikes a
better balance, achieving higher accuracy with lower latency. We show other variants in Appendix \ref{sec:changing_layers_full}.\\
\textbf{Sparsity in Attention Heads:} We discard entries below $0.01$ in attention heads. Figure~\ref{fig:ablation} shows that enforcing sparsity
improves both accuracy and runtime, sharpening the verifier’s focus on relevant attention patterns and enhancing efficiency without performance loss.

% We perform ablation studies to better understand the design choices in \ours{}, focusing on two aspects: (a) the layers used to extract model internal signals, and (b) the effect of introducing sparsity in attention heads. The results are summarized in Figure~\ref{fig:ablation}.\\
% \textbf{Changing Layers:} In Figure~\ref{fig:ablation}(a), we compare attention calculated from only the last three layers versus aggregating attention matrices across all layers. The results show that leveraging the attention of the last three layers is both more efficient and more accurate. Although using all layers slightly improves performance on some benchmarks (e.g., GSM8K), it consistently increases runtime overhead. This highlights a trade-off between marginal accuracy gains and computational efficiency, with the last three-layer strategy offering a more cost-effective solution.\\
% \textbf{Sparsity in Attention Heads:} Figure~\ref{fig:ablation}(b) examines whether enforcing sparsity in attention heads affects performance. The results indicate that the introduction of sparsity leads to both improved accuracy and reduced runtime across the benchmarks. This suggests that sparsity helps the verifier focus on more relevant attention patterns, enhancing efficiency without sacrificing performance.

\section{Conclusion and Future Work}
\label{conclusions}

In this work, we propose \ours{}, an adaptive speculative decoding that improves both efficiency and accuracy in multistep reasoning. By leveraging lightweight model-internal signals for verification, along with a self-consistency selector that identifies representative reasoning step across samples, \ours{} avoids dependence on external reward models and achieves higher accuracy with reduced latency compared to state-of-the-art methods. In future, we plan to extend \ours{} by incorporating additional internal signals such as entropy-based measures and uncertainty calibration to refine verifier reliability. 
\section{Limitations}

\label{sec:limitation}

Our work primarily focuses on improving inference-time efficiency and robustness for multi-step reasoning through verification-aware speculative decoding. While the proposed framework demonstrates consistent gains across a variety of reasoning benchmarks, we acknowledge several limitations related to the scope and evaluation of the current study.

Our experiments are conducted on a set of reasoning benchmarks that emphasize structured, step-by-step inference. Although these benchmarks cover a wide range of reasoning patterns, we do not explicitly evaluate performance on open-ended generation tasks or domains such as long-form dialogue and creative writing, where the notion of step-level verification may manifest differently. Extending the evaluation to such settings remains an interesting direction for future work.

Our analysis focuses on single-instance inference and does not explicitly address deployment aspects such as large-scale batching or hardware-specific optimizations. While our method reduces latency at the per-instance level, additional engineering considerations may be required to fully realize these benefits in production-scale systems. We leave these extensions to future work.

\section{Ethical Considerations}

Our approach relies on large language models for multi-step reasoning, it inherits known risks associated with LLMs, including hallucinated or factually incorrect outputs \cite{hallucination}. In real-world deployments, such incorrect reasoning traces or final answers could be mistakenly trusted by users, potentially leading to the dissemination of misleading or incorrect information \cite{zhang2023ethical,albrecht2022despite}. While our method aims to improve reasoning reliability by selecting the most consistent reasoning step, it does not eliminate the possibility of hallucinations, and the resulting systems should not be considered fully reliable without human oversight. The proposed approach is therefore best suited for research settings or as a supporting component in applications where appropriate validation mechanisms are in place.

Finally, our method does not rely on or introduce any private or personally identifiable data. All experiments are conducted using publicly available benchmarks. Although the underlying language models may have been pre-trained on large-scale corpora that could contain sensitive information, our method does not explicitly or implicitly solicit such content.

\bibliography{ref}

@inproceedings{hendrycks2measuring,
  title={Measuring Mathematical Problem Solving With the MATH Dataset},
  author={Hendrycks, Dan and Burns, Collin and Kadavath, Saurav and Arora, Akul and Basart, Steven and Tang, Eric and Song, Dawn and Steinhardt, Jacob},
  booktitle={Thirty-fifth Conference on Neural Information Processing Systems Datasets and Benchmarks Track (Round 2)},
  year={2021}
}

@inproceedings{rsd_paper_2025,
  title={Reward-Guided Speculative Decoding for Efficient LLM Reasoning}, 
  author={Liao, Baohao and Xu, Yuhui and Dong, Hanze and Li, Junnan and Monz, Christof and Savarese, Silvio and Sahoo, Doyen and Xiong, Caiming},
  year={2025},
  booktitle={Forty-second International Conference on Machine Learning (ICML)}
}

@inproceedings{dcscore_icml_2025,
  title={Measuring Diversity in Synthetic Datasets}, 
  author={Zhu, Y. and Zhang, H. and Wu, B. and Li, J. and Zheng, Z. and Zhao, P. and Chen, P. and Bian, Y.},
  year={2025},
  booktitle={Forty-Second International Conference on Machine Learning}
}

@inproceedings{judge_decoding_2025,
  title={Judge Decoding: Faster Speculative Sampling Requires Going Beyond Model Alignment},
  author={Bachmann, G. and Anagnostidis, S. and Pumarola, A. and Georgopoulos, M. and Sanakoyeu, A. and Du, Y. and Schönfeld, E. and Thabet, A.K. and Kohler, J.},
  booktitle={13th International Conference on Learning Representations (ICLR)},
  year={2025}
}

@inproceedings{liao2024mario,
  title={MARIO: MAth Reasoning with code Interpreter Output-A Reproducible Pipeline},
  author={Liao, Minpeng and Li, Chengxi and Luo, Wei and Jing, Wu and Fan, Kai},
  booktitle={Findings of the Association for Computational Linguistics ACL 2024},
  pages={905--924},
  year={2024}
}

@inproceedings{he2024olympiadbench,
  title={OlympiadBench: A Challenging Benchmark for Promoting AGI with Olympiad-Level Bilingual Multimodal Scientific Problems},
  author={He, Chaoqun and Luo, Renjie and Bai, Yuzhuo and Hu, Shengding and Thai, Zhen and Shen, Junhao and Hu, Jinyi and Han, Xu and Huang, Yujie and Zhang, Yuxiang and others},
  booktitle={Proceedings of the 62nd Annual Meeting of the Association for Computational Linguistics (Volume 1: Long Papers)},
  pages={3828--3850},
  year={2024}
}

@article{stern2018blockwise,
  title={Blockwise parallel decoding for deep autoregressive models},
  author={Stern, Mitchell and Shazeer, Noam and Uszkoreit, Jakob},
  journal={Advances in Neural Information Processing Systems},
  volume={31},
  year={2018}
}

@inproceedings{leviathan2023fast,
  title={Fast inference from transformers via speculative decoding},
  author={Leviathan, Yaniv and Kalman, Matan and Matias, Yossi},
  booktitle={International Conference on Machine Learning},
  pages={19274--19286},
  year={2023},
  organization={PMLR}
}

@inproceedings{xia2024unlocking,
  title={Unlocking Efficiency in Large Language Model Inference: A Comprehensive Survey of Speculative Decoding},
  author={Xia, Heming and Yang, Zhe and Dong, Qingxiu and Wang, Peiyi and Li, Yongqi and Ge, Tao and Liu, Tianyu and Li, Wenjie and Sui, Zhifang},
  booktitle={Findings of the Association for Computational Linguistics ACL 2024},
  pages={7655--7671},
  year={2024}
}

@article{chen2023accelerating,
  title={Accelerating large language model decoding with speculative sampling},
  author={Chen, Charlie and Borgeaud, Sebastian and Irving, Geoffrey and Lespiau, Jean-Baptiste and Sifre, Laurent and Jumper, John},
  journal={arXiv preprint arXiv:2302.01318},
  year={2023}
}

@inproceedings{zhang2024draft,
  title={Draft\& Verify: Lossless Large Language Model Acceleration via Self-Speculative Decoding},
  author={Zhang, Jun and Wang, Jue and Li, Huan and Shou, Lidan and Chen, Ke and Chen, Gang and Mehrotra, Sharad},
  booktitle={Proceedings of the 62nd Annual Meeting of the Association for Computational Linguistics (Volume 1: Long Papers)},
  pages={11263--11282},
  year={2024}
}

@article{sun2024triforce,
  title={TriForce: Lossless Acceleration of Long Sequence Generation with Hierarchical Speculative Decoding},
  author={Sun, Hanshi and Chen, Zhuoming and Yang, Xinyu and Tian, Yuandong and Chen, Beidi},
  journal={CoRR},
  year={2024}
}

@article{chen2024cascade,
  title={Cascade speculative drafting for even faster llm inference},
  author={Chen, Ziyi and Yang, Xiaocong and Lin, Jiacheng and Sun, Chenkai and Chang, Kevin and Huang, Jie},
  journal={Advances in Neural Information Processing Systems},
  volume={37},
  pages={86226--86242},
  year={2024}
}

@inproceedings{li2024eagle,
  title={EAGLE: Speculative Sampling Requires Rethinking Feature Uncertainty},
  author={Li, Yuhui and Wei, Fangyun and Zhang, Chao and Zhang, Hongyang},
  booktitle={International Conference on Machine Learning},
  pages={28935--28948},
  year={2024},
  organization={PMLR}
}

@inproceedings{miao2024specinfer,
  title={Specinfer: Accelerating large language model serving with tree-based speculative inference and verification},
  author={Miao, Xupeng and Oliaro, Gabriele and Zhang, Zhihao and Cheng, Xinhao and Wang, Zeyu and Zhang, Zhengxin and Wong, Rae Ying Yee and Zhu, Alan and Yang, Lijie and Shi, Xiaoxiang and others},
  booktitle={Proceedings of the 29th ACM International Conference on Architectural Support for Programming Languages and Operating Systems, Volume 3},
  pages={932--949},
  year={2024}
}

@inproceedings{fu2024break,
  title={Break the sequential dependency of LLM inference using LOOKAHEAD DECODING},
  author={Fu, Yichao and Bailis, Peter and Stoica, Ion and Zhang, Hao},
  booktitle={Proceedings of the 41st International Conference on Machine Learning},
  pages={14060--14079},
  year={2024}
}

@article{sun2023spectr,
  title={Spectr: Fast speculative decoding via optimal transport},
  author={Sun, Ziteng and Suresh, Ananda Theertha and Ro, Jae Hun and Beirami, Ahmad and Jain, Himanshu and Yu, Felix},
  journal={Advances in Neural Information Processing Systems},
  volume={36},
  pages={30222--30242},
  year={2023}
}

@article{chen2024sequoia,
  title={Sequoia: Scalable, robust, and hardware-aware speculative decoding},
  author={Chen, Zhuoming and May, Avner and Svirschevski, Ruslan and Huang, Yuhsun and Ryabinin, Max and Jia, Zhihao and Chen, Beidi},
  journal={arXiv preprint arXiv:2402.12374},
  year={2024}
}

@inproceedings{elhoushi2024layerskip,
  title={LayerSkip: Enabling Early Exit Inference and Self-Speculative Decoding},
  author={Elhoushi, Mostafa and Shrivastava, Akshat and Liskovich, Diana and Hosmer, Basil and Wasti, Bram and Lai, Liangzhen and Mahmoud, Anas and Acun, Bilge and Agarwal, Saurabh and Roman, Ahmed and others},
  booktitle={Proceedings of the 62nd Annual Meeting of the Association for Computational Linguistics (Volume 1: Long Papers)},
  pages={12622--12642},
  year={2024}
}

@article{wen2024speculative,
  title={Speculative decoding with CTC-based draft model for LLM inference acceleration},
  author={Wen, Zhuofan and Gui, Shangtong and Feng, Yang},
  journal={Advances in Neural Information Processing Systems},
  volume={37},
  pages={92082--92100},
  year={2024}
}

@inproceedings{metel2024draft,
  title={Draft on the Fly: Adaptive Self-Speculative Decoding using Cosine Similarity},
  author={Metel, Michael and Lu, Peng and Chen, Boxing and Rezagholizadeh, Mehdi and Kobyzev, Ivan},
  booktitle={Findings of the Association for Computational Linguistics: EMNLP 2024},
  pages={2267--2272},
  year={2024}
}

@inproceedings{wang2024math,
  title={Math-Shepherd: Verify and Reinforce LLMs Step-by-step without Human Annotations},
  author={Wang, Peiyi and Li, Lei and Shao, Zhihong and Xu, Runxin and Dai, Damai and Li, Yifei and Chen, Deli and Wu, Yu and Sui, Zhifang},
  booktitle={Proceedings of the 62nd Annual Meeting of the Association for Computational Linguistics (Volume 1: Long Papers)},
  pages={9426--9439},
  year={2024}
}

@article{zhou2025q,
  title={q\#: Provably optimal distributional rl for llm post-training},
  author={Zhou, Jin Peng and Wang, Kaiwen and Chang, Jonathan D and Gao, Zhaolin and Kallus, Nathan and Weinberger, Kilian Q and Brantley, Kiant{\'e} and Sun, Wen},
  journal={CoRR},
  year={2025}
}

@inproceedings{yu2024ovm,
  title={OVM, Outcome-supervised Value Models for Planning in Mathematical Reasoning},
  author={Yu, Fei and Gao, Anningzhe and Wang, Benyou},
  booktitle={Findings of the Association for Computational Linguistics: NAACL 2024},
  pages={858--875},
  year={2024}
}

@article{dong2024rlhf,
  title={RLHF Workflow: From Reward Modeling to Online RLHF A Comprehensive Practical Alignment Recipe of Iterative Preference Learning},
  author={Dong, Hanze and Xiong, Wei and Pang, Bo and Wang, Haoxiang and Zhao, Han and Zhou, Yingbo and Jiang, Nan and Sahoo, Doyen and Xiong, Caiming and Zhang, Tong},
  journal={Transactions on Machine Learning Research},
  volume={2024},
  year={2024},
  publisher={Transactions on Machine Learning Research}
}

@inproceedings{lightman2023let,
  title={Let's verify step by step},
  author={Lightman, Hunter and Kosaraju, Vineet and Burda, Yuri and Edwards, Harrison and Baker, Bowen and Lee, Teddy and Leike, Jan and Schulman, John and Sutskever, Ilya and Cobbe, Karl},
  booktitle={The Twelfth International Conference on Learning Representations},
  year={2023}
}

@article{snell2024scaling,
  title={Scaling llm test-time compute optimally can be more effective than scaling model parameters},
  author={Snell, Charlie and Lee, Jaehoon and Xu, Kelvin and Kumar, Aviral},
  journal={arXiv preprint arXiv:2408.03314},
  year={2024}
}

@article{brown2024large,
  title={Large language monkeys: Scaling inference compute with repeated sampling},
  author={Brown, Bradley and Juravsky, Jordan and Ehrlich, Ryan and Clark, Ronald and Le, Quoc V and R{\'e}, Christopher and Mirhoseini, Azalia},
  journal={arXiv preprint arXiv:2407.21787},
  year={2024}
}

@article{cobbe2021training,
  title={Training verifiers to solve math word problems},
  author={Cobbe, Karl and Kosaraju, Vineet and Bavarian, Mohammad and Chen, Mark and Jun, Heewoo and Kaiser, Lukasz and Plappert, Matthias and Tworek, Jerry and Hilton, Jacob and Nakano, Reiichiro and others},
  journal={arXiv preprint arXiv:2110.14168},
  year={2021}
}

@article{dong2023raft,
  title={RAFT: Reward rAnked FineTuning for Generative Foundation Model Alignment},
  author={Dong, Hanze and Xiong, Wei and Goyal, Deepanshu and Zhang, Yihan and Chow, Winnie and Pan, Rui and Diao, Shizhe and Zhang, Jipeng and Shum, Kashun and Zhang, Tong},
  journal={Transactions on Machine Learning Research},
  volume={2023},
  year={2023},
  publisher={Transactions on Machine Learning Research}
}

@article{chen2024alphamath,
  title={Alphamath almost zero: process supervision without process},
  author={Chen, Guoxin and Liao, Minpeng and Li, Chengxi and Fan, Kai},
  journal={Advances in Neural Information Processing Systems},
  volume={37},
  pages={27689--27724},
  year={2024}
}

@inproceedings{qimutual,
  title={Mutual Reasoning Makes Smaller LLMs Stronger Problem-Solver},
  author={Qi, Zhenting and Mingyuan, MA and Xu, Jiahang and Zhang, Li Lyna and Yang, Fan and Yang, Mao},
  booktitle={The Thirteenth International Conference on Learning Representations},
  year={2024}
}

@article{yao2023tree,
  title={Tree of thoughts: Deliberate problem solving with large language models},
  author={Yao, Shunyu and Yu, Dian and Zhao, Jeffrey and Shafran, Izhak and Griffiths, Tom and Cao, Yuan and Narasimhan, Karthik},
  journal={Advances in neural information processing systems},
  volume={36},
  pages={11809--11822},
  year={2023}
}

@inproceedings{cemri2025specs,
  title={SPECS: Faster Test-Time Scaling through Speculative Drafts},
  author={Cemri, Mert and Rajaraman, Nived and Tiwari, Rishabh and Liu, Xiaoxuan and Keutzer, Kurt and Stoica, Ion and Ramchandran, Kannan and Beirami, Ahmad and Sun, Ziteng},
  booktitle={ES-FoMo III: 3rd Workshop on Efficient Systems for Foundation Models},
  year={2025}
}

@article{brown2020language,
  title={Language models are few-shot learners},
  author={Brown, Tom and Mann, Benjamin and Ryder, Nick and Subbiah, Melanie and Kaplan, Jared D and Dhariwal, Prafulla and Neelakantan, Arvind and Shyam, Pranav and Sastry, Girish and Askell, Amanda and others},
  journal={Advances in neural information processing systems},
  volume={33},
  pages={1877--1901},
  year={2020}
}

@article{hurst2024gpt,
  title={Gpt-4o system card},
  author={Hurst, Aaron and Lerer, Adam and Goucher, Adam P and Perelman, Adam and Ramesh, Aditya and Clark, Aidan and Ostrow, AJ and Welihinda, Akila and Hayes, Alan and Radford, Alec and others},
  journal={arXiv preprint arXiv:2410.21276},
  year={2024}
}

@article{team2024gemini,
  title={Gemini 1.5: Unlocking multimodal understanding across millions of tokens of context},
  author={Team, Gemini and Georgiev, Petko and Lei, Ving Ian and Burnell, Ryan and Bai, Libin and Gulati, Anmol and Tanzer, Garrett and Vincent, Damien and Pan, Zhufeng and Wang, Shibo and others},
  journal={arXiv preprint arXiv:2403.05530},
  year={2024}
}

@article{lin2024awq,
  title={Awq: Activation-aware weight quantization for on-device llm compression and acceleration},
  author={Lin, Ji and Tang, Jiaming and Tang, Haotian and Yang, Shang and Chen, Wei-Ming and Wang, Wei-Chen and Xiao, Guangxuan and Dang, Xingyu and Gan, Chuang and Han, Song},
  journal={Proceedings of machine learning and systems},
  volume={6},
  pages={87--100},
  year={2024}
}

@inproceedings{xuthink,
  title={ThinK: Thinner Key Cache by Query-Driven Pruning},
  author={Xu, Yuhui and Jie, Zhanming and Dong, Hanze and Wang, Lei and Lu, Xudong and Zhou, Aojun and Saha, Amrita and Xiong, Caiming and Sahoo, Doyen},
  booktitle={The Thirteenth International Conference on Learning Representations},
  year={2024}
}

@article{patterson2021carbon,
  title={Carbon emissions and large neural network training},
  author={Patterson, David and Gonzalez, Joseph and Le, Quoc and Liang, Chen and Munguia, Lluis-Miquel and Rothchild, Daniel and So, David and Texier, Maud and Dean, Jeff},
  journal={arXiv preprint arXiv:2104.10350},
  year={2021}
}

@inproceedings{holtzmancurious,
  title={The Curious Case of Neural Text Degeneration},
  author={Holtzman, Ari and Buys, Jan and Du, Li and Forbes, Maxwell and Choi, Yejin},
  booktitle={International Conference on Learning Representations},
  year={2020}
}

@misc{qwen2025qwen25technicalreport,
      title={Qwen2.5 Technical Report}, 
      author={Qwen and : and An Yang and Baosong Yang and Beichen Zhang and Binyuan Hui and Bo Zheng and Bowen Yu and Chengyuan Li and Dayiheng Liu and Fei Huang and Haoran Wei and Huan Lin and Jian Yang et al.},
      year={2025},
      eprint={2412.15115},
      archivePrefix={arXiv},
      primaryClass={cs.CL},
      url={https://arxiv.org/abs/2412.15115}, 
}

@article{yang2024qwen2,
  title={Qwen2. 5-math technical report: Toward mathematical expert model via self-improvement},
  author={Yang, An and Zhang, Beichen and Hui, Binyuan and Gao, Bofei and Yu, Bowen and Li, Chengpeng and Liu, Dayiheng and Tu, Jianhong and Zhou, Jingren and Lin, Junyang and others},
  journal={arXiv preprint arXiv:2409.12122},
  year={2024}
}

@article{dubey2024llama,
  title={The llama 3 herd of models},
  author={Dubey, Abhimanyu and Jauhri, Abhinav and Pandey, Abhinav and Kadian, Abhishek and Al-Dahle, Ahmad and Letman, Aiesha and Mathur, Akhil and Schelten, Alan and Yang, Amy and Fan, Angela and others},
  journal={arXiv e-prints},
  pages={arXiv--2407},
  year={2024}
}

@misc{skywork,
      title={Skywork-o1 open series}, 
      author={o1 Team, S.},
      year={2024},
      url={https://huggingface.co/Skywork.}
}

@article{pan2025specreason,
  title={Specreason: Fast and accurate inference-time compute via speculative reasoning},
  author={Pan, Rui and Dai, Yinwei and Zhang, Zhihao and Oliaro, Gabriele and Jia, Zhihao and Netravali, Ravi},
  journal={arXiv preprint arXiv:2504.07891},
  year={2025}
}

@article{hallucination,
	doi = {10.1145/3571730},
  
	url = {https://doi.org/10.1145%2F3571730},
  
	year = 2023,
	month = {mar},
  
	publisher = {Association for Computing Machinery ({ACM})},
  
	volume = {55},
  
	number = {12},
  
	pages = {1--38},
  
	author = {Ziwei Ji and Nayeon Lee and Rita Frieske and Tiezheng Yu and Dan Su and Yan Xu and Etsuko Ishii and Ye Jin Bang and Andrea Madotto and Pascale Fung},
  
	title = {Survey of Hallucination in Natural Language Generation},
  
	journal = {{ACM} Computing Surveys}
}

@misc{zhang2023ethical,
      title={Ethical Considerations and Policy Implications for Large Language Models: Guiding Responsible Development and Deployment}, 
      author={Jianyi Zhang and Xu Ji and Zhangchi Zhao and Xiali Hei and Kim-Kwang Raymond Choo},
      year={2023},
      eprint={2308.02678},
      archivePrefix={arXiv},
      primaryClass={cs.CY}
}

@misc{albrecht2022despite,
      title={Despite "super-human" performance, current LLMs are unsuited for decisions about ethics and safety}, 
      author={Joshua Albrecht and Ellie Kitanidis and Abraham J. Fetterman},
      year={2022},
      eprint={2212.06295},
      archivePrefix={arXiv},
      primaryClass={cs.CL}
}

\appendix
\clearpage
\newpage

\section{Appendix}
\subsection{Datasets Description}
\label{sec:datasets}

An overview of the dataset statistics and examples are shown in Table \ref{tab:datasets_overview}.

\begin{table*}[b]
\caption{Overview of the Complex QA datasets used in this study.}
\label{tab:datasets_overview}
\resizebox{\textwidth}{!}{

\begin{tabular}{lcll}

\hline
\textbf{Dataset} & \textbf{\#Test} & \textbf{Example Question} & \textbf{Description} \\ \hline

\colorg       & \colorg     & \colorg \texttt{What is the smallest positive perfect cube}      & multi-step  \colorg   \\
\colorg \textbf{MATH500}~\cite{hendrycks2measuring}& \colorg  \colorg 500 & \colorg  \texttt{that can be written as the sum of three}  & arithmetic word  \colorg    \\
\colorg  & \colorg & \colorg \texttt{consecutive integers?} & problems \colorg  \\

 &       &  \texttt{The red car is 40\% cheaper than the blue car.}      &   multi-step  \\
\textbf{GSM8K}~\cite{cobbe2021training} & 1319 &  \texttt{The price of the blue car is \$100. How much} & arithmetic word  \\
&   & \texttt{do both cars cost?} & problems  \\

\colorg       & \colorg   & \colorg  \texttt{Suppose the universe set is U=\{0,1,2,4,6,8\}.}      & \colorg multi-step   \\
\colorg \textbf{GaoKao-2023-En}~\cite{liao2024mario} & \colorg 385 & \colorg  \texttt{Two of its subsets are M=\{0,4,6\}, N=\{0,1,6\}.} &\colorg arithmetic word     \\
\colorg &  \colorg  & \colorg \texttt{Find $M \cup \bar{N}$.} &\colorg problems \\

 &       & \texttt{A number is called Norwegian if it has three}      & multi-step  \\
\textbf{OlympiadBench}~\cite{he2024olympiadbench}  & 675 & \texttt{distinct positive divisors whose sum is equal}      & arithmetic word  \\
&  & \texttt{to 2022. Determine smallest Norwegian number.}      & problems
\\

\bottomrule
\end{tabular}
}

\end{table*}

\textbf{MATH500}: A benchmark subset curated from the MATH dataset, consisting of 500 competition-level mathematics problems spanning algebra, geometry, combinatorics, number theory, and probability. Each problem is accompanied by a detailed step-by-step solution, requiring multi-hop symbolic and logical reasoning. We use the full 500 problems as the evaluation set.

\textbf{GSM8K}: A dataset of linguistically diverse grade-school math word problems designed to test multi-step numerical reasoning. It comprises 8.5K questions, with a test set of 1,319 problems. Each question includes annotated solutions with intermediate steps, encouraging models to demonstrate faithful reasoning chains.

\textbf{Gaokao-2023-En}: Derived from the English-translated 2023 Gaokao (China’s national college entrance exam), this dataset contains high-school level math word problems with a strong emphasis on reasoning over algebra, functions, and applied mathematics. It poses particular challenges due to its formal problem style and complex solution trajectories. The evaluation set includes 385 problems.

\textbf{OlympiadBench}: A large-scale benchmark of problems drawn from global mathematics and science Olympiads, covering topics such as advanced algebra, geometry, physics, and logical reasoning. The problems are highly challenging, requiring creative multi-step reasoning far beyond routine computation. We evaluate on the test split of 675 questions.

\subsection{Additional Experiments}
\label{sec:prompts}

\subsubsection{Changing Layers}
\label{sec:changing_layers_full}
We perform ablation studies to analyze key design choices in \ours{}, particularly the selection of layers used to extract internal grounding signals. Figure~\ref{fig:change_layers_full} compares four settings: using attention from (1) all layers, (2) the last three layers, (3) the middle three layers, and (4) the first three layers. The results show that the first three layers perform noticeably worse than other variants, while the middle three layers achieve higher accuracy but still lag behind the deeper layers. Both the last three layers and all layers yield strong performance, but using all layers incurs higher runtime overhead. Overall, the last three layers provide the best trade-off, delivering strong accuracy with lower latency.

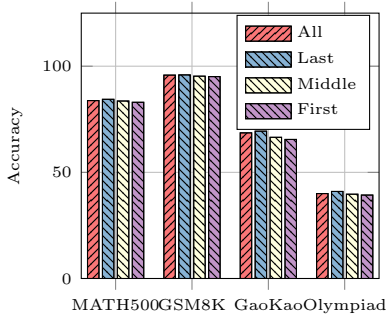
\begin{figure}[!t]
\centering
\begin{tikzpicture}

    \begin{axis}[
            ybar=1.3pt,
            width=1\textwidth,
            bar width=0.15,
            height=5.1cm,
            width=5.5cm,
            every axis plot/.append style={fill},
            grid=major,
            xtick={1, 2, 3, 4},
            xticklabels={MATH500, GSM8K, GaoKao, Olympiad},
            ylabel style = {font=\tiny},
        yticklabel style = {font=\tiny,xshift=0.05ex},
        xticklabel style ={font=\tiny,yshift=0.1ex},
            ylabel={Accuracy},
            enlarge x limits=0.15,
            ymin=0,
            ymax=125,
            legend style ={font=\tiny,yshift=0.5ex},
            area legend,
            nodes near coords style={font=\tiny,align=center,text width=2em},
            legend entries={All, Last, Middle, First},
            legend cell align={left},
            legend pos = {north east},
            legend style={/tikz/every even column/.append style={column sep=0.01cm}},
        ]
        \addplot+[
            ybar,
            plotColor1*,
            nodes near coords align={vertical},
            draw=black,
            postaction={
                    pattern=north east lines
                },
        ] plot coordinates { 
                (1,83.8)
                (2,95.8)
                (3,68.6)
                (4,40)
            };
        \addplot+[
            ybar,
            plotColor2*,
            draw=black,
    nodes near coords align={vertical},
            postaction={
                    pattern=north west lines
                },
        ] plot coordinates { 
                (1,84.4)
                (2,95.9)
                (3,69.4)
                (4,41)
            };
 
        \addplot+[
            ybar,
            plotColor5*,
            draw=black,
    nodes near coords align={vertical},
            postaction={
                    pattern=north west lines
                },
        ] plot coordinates { 
                (1,83.6)
                (2,95.3)
                (3,66.5)
                (4,39.7)
            };

                    \addplot+[
            ybar,
            plotColor4*,
            draw=black,
    nodes near coords align={vertical},
            postaction={
                    pattern=north west lines
                },
        ] plot coordinates { 
                (1,83)
                (2,95.1)
                (3,65.5)
                (4,39.3)
            };
            
    \end{axis}

\end{tikzpicture}
\caption{Changing Layers}
\label{fig:change_layers_full}
\end{figure}

\subsubsection{Tuning of $\beta$ and $\tau$}
\label{sec:hyperparameter}
We analyze the sensitivity of our approach to two hyperparameters: the step acceptance threshold $\tau$ for our ensemble verifier, and the weighting factor $\beta$, which balances the log-probability and attention grounding score when computing the ensemble score. As shown in Table~\ref{tab:beta}, accuracy remains stable across different values of $\beta$, with $\beta=0.3$ performing slightly better than higher values, suggesting that a moderate weighting strikes a good balance between model confidence and grounding. Table~\ref{tab:tau} reports results for varying $\tau$. We find that $\tau=0.7$ achieves the most consistent gains across datasets, while both lower ($\tau=0.6$) and higher values ($\tau=0.8, 0.9$) lead to small drops. Overall, our method is robust to hyperparameter choices, with $\beta=0.3$ and $\tau=0.7$ serving as effective defaults across tasks.

\subsubsection{Qualitative Analysis}
\label{sec:qualitative}
Table~\ref{tab:exemplar_qualitative_1}, \ref{tab:exemplar_qualitative_2} presents a qualitative example of reasoning steps scored by the PRM. Each intermediate step receives a high score, leading the verifier to accept the draft-generated reasoning without intervention. However, despite this consistent acceptance, the reasoning chain ultimately produces an incorrect final answer. This illustrates a key limitation of relying solely on PRM scores: while they may capture local plausibility of individual steps, they do not guarantee global correctness of the overall solution. Such cases highlight the need for more robust evaluation mechanisms that can account for consistency across steps as well as correctness of the final outcome.

\begin{table*}[!t]
\centering
\caption{Accuracy with different $\beta$s. Overall, $\beta$ = 0.3 works well for different tasks.}
\label{tab:beta}
\resizebox{\textwidth}{!}{
\begin{tabular}{lcccccccc}
\toprule
\textbf{Method} & \textbf{Target Model} & \textbf{Draft Model} & \textbf{Setting}   &  \textbf{MATH500} & \textbf{GSM8K} & \makecell{\textbf{Gaokao} \\ \textbf{2023 En}} & \makecell{\textbf{Olympiad} \\ \textbf{Bench}} \\

\midrule

\multicolumn{8}{c}{\textbf{Math Model, Draft and Target: Qwen2.5-Math-Instruct}} \\
\midrule
Ours Majority   & 7B & 1.5B  & $\beta = 0.3$ & 85.4 & 95.8 &  69.4 & 41.2  \\
Ours Majority  & 7B  & 1.5B & $\beta = 0.5$ & 85.0 & 95.7 & 68.5 & 40.4 \\
Ours Majority  & 7B  & 1.5B  & $\beta = 0.7$ & 84.4 & 95.6 & 65.5 & 40.2 \\

\bottomrule
\end{tabular}}
\end{table*}

\begin{table*}[!t]
\centering
\caption{Accuracy with different $\tau$s. Overall, $\tau$ = 0.7 works well for different tasks.}
\label{tab:tau}
\resizebox{\textwidth}{!}{
\begin{tabular}{lcccccccc}
\toprule
\textbf{Method} & \textbf{Target Model} & \textbf{Draft Model} & \textbf{Setting}   &  \textbf{MATH500} & \textbf{GSM8K} & \makecell{\textbf{Gaokao} \\ \textbf{2023 En}} & \makecell{\textbf{Olympiad} \\ \textbf{Bench}} \\

\midrule

\multicolumn{8}{c}{\textbf{Math Model, Draft and Target: Qwen2.5-Math-Instruct}} \\
\midrule
Ours Majority  & 7B  & 1.5B & $\tau = 0.6$ & 83.6 & 93.5 & 67.4 & 39.7 \\
Ours Majority  & 7B  & 1.5B  & $\tau = 0.7$ & 85.4 & 95.8 & 69.4 & 41.2 \\
Ours Majority  & 7B  & 1.5B  & $\tau = 0.8$ & 84.2 & 94.6 &  68.7 & 40.4 \\
Ours Majority  & 7B  & 1.5B  & $\tau = 0.9$ & 85.1 & 95.6 &  69.2 & 41.0 \\

\bottomrule
\end{tabular}}
\end{table*}

\begin{table*}[!ht]
\caption{Qualitative analysis of reasoning steps scored by \textbf{PRM}. \textit{Even though all draft-generated steps receive high acceptance scores from the PRM}, \textcolor{red}{the final answer is still incorrect}.}
\label{tab:exemplar_qualitative_1}
\begin{tabular}{p{0.75\linewidth}c}
\toprule
\multicolumn{2}{p{\linewidth}}{\textit{Question: A treasure hunter found a buried treasure chest filled with gems. There were 175 diamonds, 35 fewer rubies than diamonds, and twice the number of emeralds than the rubies. How many of the gems were there in the chest?}} \\
\toprule
    \textbf{Reasoning Steps} & \textbf{PRM Scores} \\
\midrule
% \small
Step 1: Identify the number of diamonds in the chest.
The treasure hunter found 175 diamonds.  & 0.962108
 \\\cline{1-2}
Step 2: Calculate the number of rubies.
There were 35 fewer rubies than diamonds, so the number of rubies is 175 - 35 = 140. & 0.962108
 \\\cline{1-2}
Step 3: Calculate the number of emeralds.
The number of emeralds is twice the number of rubies, so the number of emeralds is 2 * 140 = 280. & 0.962108
 \\\cline{1-2}
Step 4: Calculate total number of gems.
To find the total number of gems, add the number of diamonds, rubies, and emeralds: 175 + 140 + 280 = 535. & 0.961534
 \\\cline{1-2}
Step 5: The final answer is: \textcolor{red}{\textbackslash boxed\{535\}} & 0.962108
\\
\midrule
\end{tabular}

%\vspace{-0.25cm}
\end{table*}

\begin{table*}[!ht]
\caption{Qualitative analysis of reasoning steps scored by \textbf{PRM}. \textit{Even though all draft-generated steps receive high acceptance scores from the PRM}, \textcolor{red}{the final answer is still incorrect}.}
\label{tab:exemplar_qualitative_2}
\begin{tabular}{p{0.75\linewidth}c}
\toprule
\multicolumn{2}{p{\linewidth}}{\textit{Question: Two girls each got 1/6 of the 24 liters of water. Then a boy got 6 liters of water. How many liters of water were left?}} \\
\toprule
    \textbf{Reasoning Steps} & \textbf{PRM Scores} \\
\midrule
% \small
Step 1: Calculate the amount of water each girl received.
Each girl got 1/6 of 24 liters of water, which is (1/6) * 24 = 4 liters.  & 0.92192
 \\\cline{1-2}
Step 2: Calculate the total amount of water the girls received.
Since there are two girls, the total amount of water they received is 2 * 4 = 8 liters. & 0.92079
 \\\cline{1-2}
Step 3: Calculate the amount of water left after the boy got 6 liters.
Subtract the amount of water the boys received from the total amount of water the girls received: 8 - 6 = 2 liters. & 0.92079
 \\\cline{1-2}
Step 4: The final answer is: \textcolor{red}{\textbackslash boxed\{2\}} & 0.92079
\\
\midrule
\end{tabular}

%\vspace{-0.25cm}
\end{table*}

\begin{table*}[t]
\caption{Qualitative analysis of reasoning steps accepted by \textbf{\ours{}}, where \textit{consistent step verification leads to the} \textcolor{green}{correct final answer}.}
\label{tab:exemplar_qualitative_ours_1}
\begin{tabular}{p{0.99\linewidth}}
\toprule
\textit{Question: A treasure hunter found a buried treasure chest filled with gems. There were 175 diamonds, 35 fewer rubies than diamonds, and twice the number of emeralds than the rubies. How many of the gems were there in the chest?} \\
\toprule
    \textbf{Reasoning Steps}  \\
\midrule
Step 1: First, let's identify the number of diamonds, which is given as 175.\\\cline{1-1}
Step 2: The number of rubies is 35 fewer than the number of diamonds, so we need to subtract 35 from 175 to find the number of rubies: 175 - 35 = 140 rubies.\\\cline{1-1}
Step 3: The number of emeralds is twice the number of rubies, so we need to multiply the number of rubies by 2 to find the number of emeralds: 140 * 2 = 280 emeralds.\\\cline{1-1}
Step 4: To find the total number of gems in the chest, we need to add the number of diamonds, rubies, and emeralds: 175 + 140 + 280 = 595 gems.\\\cline{1-1}
Step 5: The final answer is: \textcolor{green}{\textbackslash boxed\{595\}}
\\
\midrule
\end{tabular}
%\vspace{-0.25cm}
\end{table*}

\begin{table*}[t]
\caption{Qualitative analysis of reasoning steps accepted by \textbf{\ours{}}, where \textit{consistent step verification leads to the} \textcolor{green}{correct final answer}.}
\label{tab:exemplar_qualitative_ours_2}
\begin{tabular}{p{0.99\linewidth}}
\toprule
\textit{Question: Two girls each got 1/6 of the 24 liters of water. Then a boy got 6 liters of water. How many liters of water were left?} \\
\toprule
\textbf{Reasoning Steps} \\
\midrule
Step 1: First, let's calculate the amount of water each girl got. Since they each got 1/6 of the 24 liters, we can calculate this as 24 * (1/6) = 4 liters per girl.\\\cline{1-1}

Step 2: Since there are two girls, the total amount of water they got is 4 * 2 = 8 liters.\\\cline{1-1}

Step 3: Now, let's calculate the amount of water left after the girls got their share. We subtract the amount they got from the total amount of water: 24 - 8 = 16 liters.\\\cline{1-1}

Step 4: Then, the boy got 6 liters of water. We subtract this from the remaining water: 16 - 6 = 10 liters.\\\cline{1-1}

Step 5: The final answer is: \textcolor{green}{\textbackslash boxed\{10\}} \\
\midrule
\end{tabular}
\end{table*}

%%%%%%%%%%%%%%%%%%%%%%%%%%%%%%%%%%%%%
%%%% Complexity Analysis
%%%%%%%%%%%%%%%%%%%%%%%%%%%%%%%%%%%%%

\subsection{Complexity Analysis - \ours{}}
\label{complexity_analysis}
We compare the computational complexity of
Speculative Decoding (SD), Reward-guided Speculative Decoding (RSD), and \ours{}. 

Let: 
\begin{itemize}
    \item $T$ = number of reasoning steps,
    \item $k$ = number of draft candidates sampled per step,
    \item $d$ = hidden dimension of embeddings,
    \item $H$ = number of attention heads,
    \item $L$ = number of layers used in ABGV,
    \item $C_{\text{draft}}, C_{\text{target}}$ = per-step cost of draft and target models.
\end{itemize}

\paragraph{Speculative Decoding (SD).} 

\text{Complexity} 
 \[= O\!\left(T \cdot (k \cdot C_{\text{draft}} + (1-\pi) \cdot k \cdot C_{\text{target}})\right),\]

where $\pi$ is the acceptance probability. 

\paragraph{Reward-guided Speculative Decoding (RSD).} 
In addition to SD cost, RSD requires a pretrained reward model (PRM) verifier:

\text{Complexity} 
= \[O\!\left(T \cdot (k \cdot C_{\text{draft}} + (1-\pi) \cdot k \cdot C_{\text{target}} + C_{\text{PRM}})\right).\]

\paragraph{\ours{}.} 
It includes draft sampling, self-consistency
selector, and ensemble verification.

\text{Complexity} 
= $O\!\left(T \cdot \big(k \cdot C_{\text{draft}} + k^2 d 
+ \tilde{L} \tilde{H} T^2 + (1-\pi) \cdot k \cdot C_{\text{target}}\big)\right),$

where $\tilde{L} \ll L$ and $\tilde{H} \ll H$ represent the reduced number of layers and heads
used in ABGV under sparsity/last-layer approximations. 

\noindent
It is not difficult to infer the following from the complexity analysis.
\begin{itemize}
    \item \ours{} avoids the external PRM cost $C_{\text{PRM}}$ in RSD, 
          reducing latency and compute. 
    \item With practical optimizations ($\tilde{L} \approx 3$, $\tilde{H} \ll H$), the ABGV overhead is negligible relative to the forward draft / target passes. 
    \item Empirically, \ours{} achieves up to $11\%$ lower runtime than RSD while improving accuracy by 1--3\% in benchmarks. 
\end{itemize}

\subsection{LLM Usage}
Large Language Models (LLMs) were used in this work solely as general-purpose assistive tools. Specifically, they were employed in two limited capacities: (i) to aid in polishing the writing for clarity and readability, and (ii) to assist in retrieval and discovery tasks, such as identifying related work. No part of the research design, algorithm development, theoretical analysis, or experimental implementation relied on LLMs. Their role was restricted to supportive tasks.

\end{document}